\documentclass[11pt]{article}
 
\usepackage{graphicx}
 \usepackage{times}
  \usepackage{amsfonts,amssymb,mathrsfs,amsmath,amssymb,amsthm,float}

  \usepackage{epstopdf}
 \usepackage{geometry}
 
\usepackage{bm}
\usepackage{subfigure}
\usepackage{amsfonts, amssymb}
\usepackage{color}
\usepackage{multicol, multirow}
\usepackage{cite}
\usepackage{mathtools}
\usepackage{thmtools, thm-restate}
\usepackage{booktabs}

\newtheorem{theorem}{Theorem}[section]
\newtheorem{lemma}[theorem]{Lemma}
\newtheorem{corollary}[theorem]{Corollary}
\newtheorem{definition}[theorem]{Definition}

\def\R{{\mathbb{R}}}

\def\Z{{\mathbb{Z}}}

\def\D{{\mathbb{D}}}
\def\B{{\mathbb{B}}}

\def\H{{\mathcal{H}}}

\def\modd{{\,\hbox{\rm{mod}}\,}}

\def\mymax{{\hbox{\rm{max}}}}
\def\myargmax{{\hbox{\rm{argmax}}}}

\def\zero{{\mathbf{0}}}
\def\one{{\mathbf{1}}}

\begin{document}
\title{Restore Translation Using Equivariant Neural Networks
}
%
%
\author{Yihan Wang, Lijia Yu, Xiao-Shan Gao\\
Academy of Mathematics and Systems Science, Chinese Academy of Sciences\\
University of  Chinese Academy of Sciences}
\maketitle

\begin{abstract}
\noindent
Invariance to spatial transformations such as translations and rotations
is a desirable property and a basic design principle for classification neural networks.
However, the commonly used convolutional neural networks (CNNs) are actually very sensitive to even small translations.
There exist vast works to achieve exact or approximate transformation invariance
by designing transformation-invariant models or assessing the transformations.
These works usually make changes to the standard CNNs
and harm the performance on standard datasets.
In this paper, rather than modifying the classifier,
we propose a pre-classifier restorer
to recover translated (or even rotated) inputs to the original ones
which will be fed into  any  classifier for the same dataset.
The restorer is based on a theoretical result which
gives a sufficient and necessary condition
for an affine operator to be translational equivariant on a tensor space.
%
\end{abstract}

\section{Introduction}
\label{sec:intro}

Deep convolutional neural networks (CNNs) had outperformed humans in many computer vision tasks~\cite{lecun1998gradient,he2016deep}.
%
One of the key  ideas in designing the CNNs is that the
convolution layer is equivariant with respect to translations,
which was emphasized both in the earlier work~\cite{fukushima1982neocognitron}
and the modern CNN~\cite{lecun1998gradient}.
However, the commonly used components, such as pooling~\cite{gholamalinezhad2020pooling}
and dropout~\cite{srivastava2013improving, srivastava2014dropout},
which help the network  to extract features and generalize,
actually make CNNs not equivariant to even small translations, as
pointed out in~\cite{azulay2018deep, engstrom2018rotation}.
As a comprehensive evaluation,
Figure~\ref{fig:acc-reduction} shows that
two classification CNNs suffer
the accuracy reductions
of more than $11\%$ and $59\%$ respectively on CIFAR-10 and MNIST,
when the inputs are horizontally and vertically translated at most 3 pixels.

\begin{figure}[htbp]
\centering
    \includegraphics[width=0.49\textwidth]{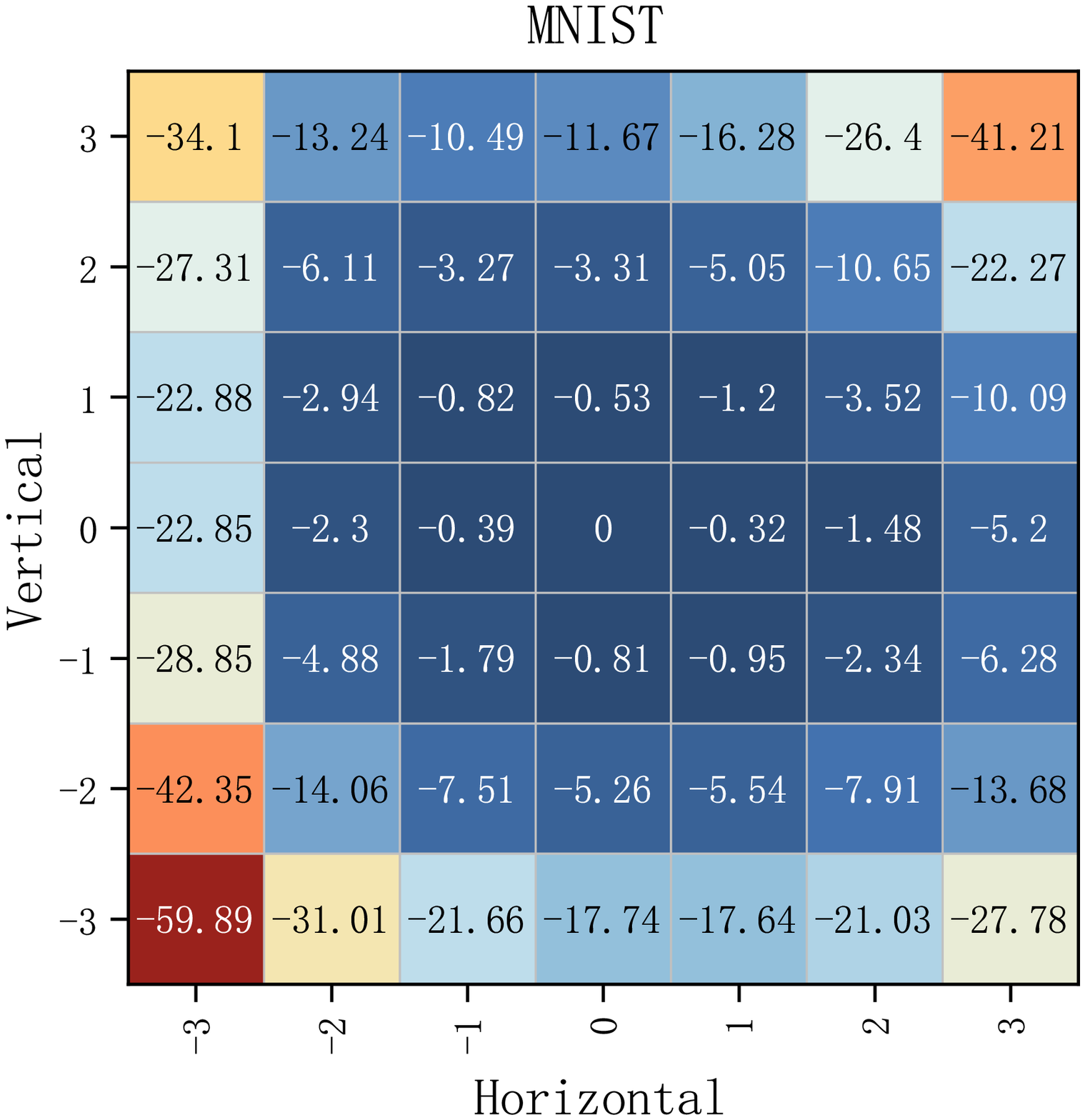}
    \includegraphics[width=0.49\textwidth]{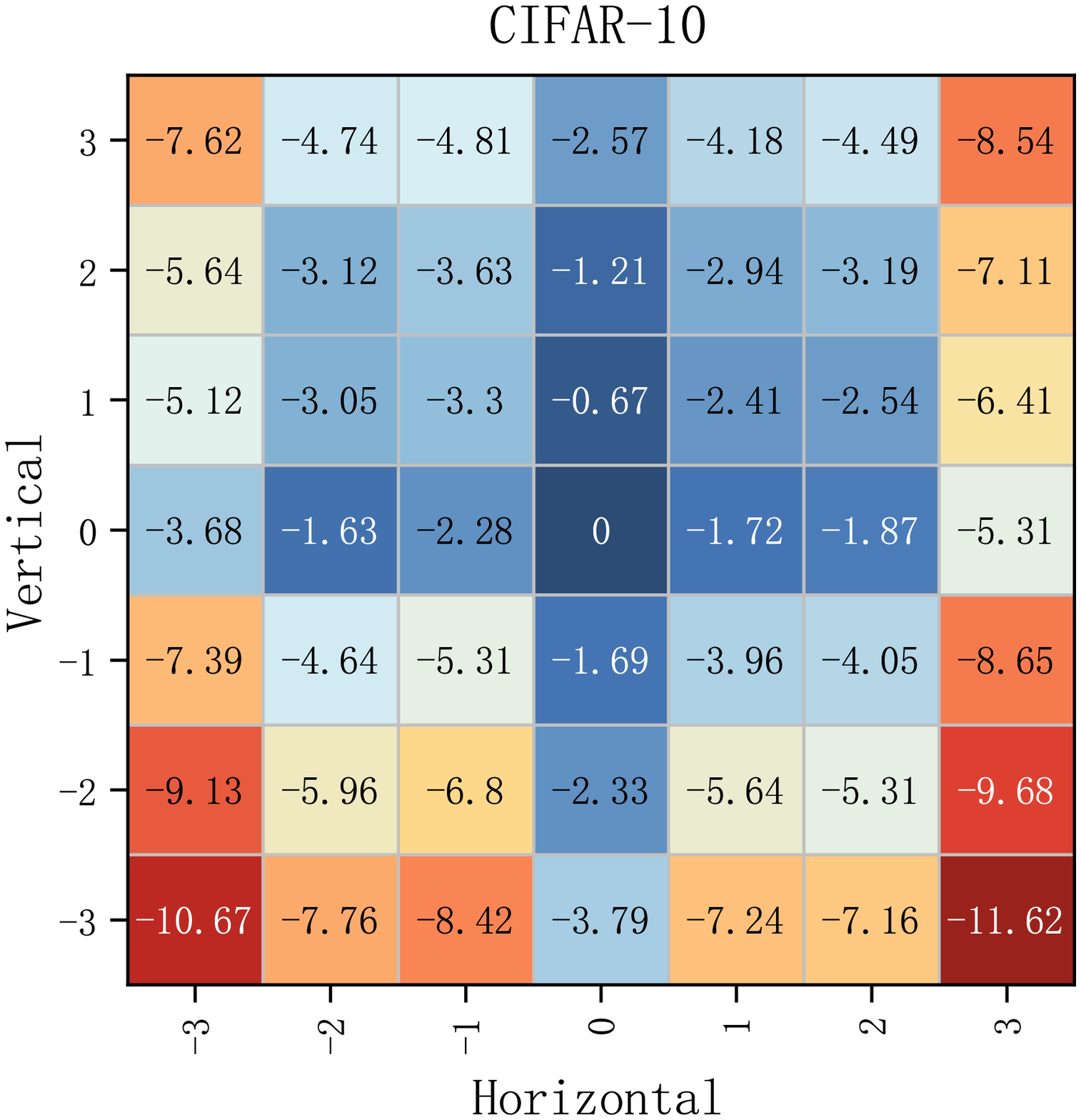}
\caption{
The accuracy reduction after vertical and horizontal translations.
The translation scope is [-3, 3] pixels.
Left: LeNet-5 on MNIST; Right: VGG-16 on CIFAR-10.
}
\label{fig:acc-reduction}
\end{figure}

Invariance to spatial transformations, including translations, rotations
and scaling, is a desirable property for classification neural networks
and the past few decades have witnessed thriving explorations on this topic.
In general, there exist three ways to achieve exact or approximate invariance.
The first is to design transformation-invariant
neural network structures~\cite{cohen2016group,gens2014deep,xu2014scale,
    marcos2018scale, ghosh2019scale, henriques2017warped,shen2016transform,
    sohn2012learning}.
%
The second is to assess and approximate transformations
via a learnable module~\cite{jaderberg2015spatial, esteves2017polar}
and then use the approximation to reduce the transformed inputs to ``standard'' ones.
The third is data augmentation~\cite{engstrom2018rotation, azulay2018deep, shorten2019survey}
by adding various transformations of the samples in the original dataset.

Those ad-hoc architectures to achieve invariance often bring extra parameters
but harm the network performance on standard datasets.
Moreover, the various designs with different purposes are not compatible with each other.
Data augmentation is not a scalable method since the invariance
that benefits from a certain augmentation protocol does not generalize to
other transformations~\cite{azulay2018deep}.
Including learnable modules such as the Spatial Transformer,
all the three approaches require training the classifier from scratch
and fail to endow existing trained networks with some invariance.
It was indicated in \cite{azulay2018deep} that ``the problem of insuring invariance to small image transformations in neural networks while preserving high accuracy remains unsolved.''

In this paper,
rather than designing any in-classifier component to make the classifier
invariant to some transformation,
we propose a pre-classifier restorer
to restore translated or rotated inputs to the original ones.
The invariance is achieved by feeding the restored inputs into any following classifier.
Our restorer depends only on the dataset instead the classifier.
Namely, the training processes of the restore and classifier are separate and
a restore is universal to any classifier trained on the same dataset.

We split the whole restoration into two stages,
transformation estimation and inverse transformation, see Figure~\ref{fig:arch}.
In the first stage, we expect that standard inputs lead to standard outputs
and the outputs of translated inputs reflect the translations.
Naturally, what we need is a strictly translation-equivariant neural network.
In Section~\ref{sec:equi}, we investigate at the theoretical aspect
the sufficient and necessary condition to construct
a strictly equivariant affine operator on a tensor space.
The condition results in \emph{the circular filters}, see Definition~\ref{def:circular-filter},
as the fundamental module
to a strictly translation-equivariant neural network.
We give the canonical architecture of translation-equivariant networks, see Equation~\eqref{equ:network}.
In Section~\ref{sec:restorer}, details of the restorer are presented.
We define a translation estimator, the core component of a restorer,
as a strictly translation-equivariant neural network
that guarantees the first component of every output on a dataset
to be the largest component, see Definition~\ref{def:estimator}.
For a translated input, due to the strict equivariance,
the largest component of the output reflect the translation.
Thus we can translate it inversely in the second stage and obtain the original image.
Though the restorer is independent on the following classifier,
it indeed depends on the dataset.
Given a dataset satisfying some reasonable conditions, i.e. \emph{an aperiodic finite dataset},
see Definition~\ref{def:aperiodic},
we prove the existence of a translation estimator, i.e. a restorer,
with the canonical architecture for this dataset.
Moreover, rotations can be viewed as translations
by converting the Cartesian coordinates to polar coordinates
and the rotation restorer arises in the similar way.

In Section~\ref{sec:exp}, the experiments on MNIST, 3D-MNIST and CIFAR-10
show that our restorers not only visually restore the translated   inputs
but also largely eliminate the accuracy reduction phenomenon.

\section{Related works} \label{sec:relate}

As generalization of convolutional neural networks,
group-equivariant convolutional neural networks~\cite{cohen2016group,gens2014deep}
exploited symmetries to endow networks with invariance to some group actions,
such as the combination of translations
and rotations by certain angles.
The warped convolutions~\cite{henriques2017warped} converted some other spatial transformations
into translations and thus obtain equivariance to these spatial transformations.
Scale-invariance~\cite{xu2014scale, marcos2018scale, ghosh2019scale} was injected into networks
by some ad-hoc components.
Random transformations~\cite{shen2016transform} of features maps were introduced in order to
prevent the dependencies of network outputs on specific poses of inputs.
Similarly, probabilistic max pooling~\cite{sohn2012learning} of the hidden units over the set of transformations
improved the invariance of networks in unsupervised learning.
Moreover, local covariant feature detecting methods~\cite{lenc2016learning, zhang2017learning} were
proposed to address the problem of extracting viewpoint invariant features from images.

Another approach to achieve invariance
is ``shiftable'' down-sampling~\cite{lenc2015understanding},
in which any original pixel can be linearly interpolated from the pixels
on the sampling grid.
These ``shiftable'' down-sampling exists if and only if the sampling frequency is
at least twice the highest frequency of the unsampled signal.

The Spatial Transformer~\cite{jaderberg2015spatial, esteves2017polar}, as a learnable module,
produces a predictive transformation for each input image
and then spatially transforms the input to a canonical pose to simplify the inference in the
subsequent layers.
Our restorers give input-specific transformations as well and adjust the input
to alleviate the poor invariance of the following classifiers.
Although the Spatial Transformers and our restorer are both learnable modules,
the training of the former depend not only on data but also on the subsequent layers,
while the latter are independent of the subsequent classifiers.

\section{Equivariant neural networks} \label{sec:equi}
Though objects in nature have continuous properties,
once captured and converted to digital signals,
there properties are represented by real tensors.
In this section,
we study the equivariance of operators on a tensor space.

\subsection{Equivariance in tensor space} \label{subsec:equi-tensor}
Assume that a map $\tilde{x}:\R^d\to \D$ stands for a property of some $d$-dimensional object
where $\D\subseteq \R$.
Sampling $\tilde{x}$ over a $(n_1, n_2, \cdots, n_d)$-grid
results in a tensor $x$ in a tensor space
\begin{equation}
    \H\coloneqq \D^{n_1}\otimes\D^{n_2}\otimes \cdots \otimes \D^{n_d}.
\end{equation}
We denote $[n]=[0,1,\ldots,n-1]$ for $n\in\Z_+$ and assume $k\modd n\in[n]$ for $k\in\Z$.
For an index $I=(i_1, i_2, \cdots, i_d)\in \prod_{i=1}^d [n_i]$ and $x\in\H$,
denote $x[I]$ to be the element of $x$ with subscript $(i_1, i_2, \cdots, i_d)$.
For convenience, we extend the index of $\H$ to  $I=(i_1, i_2, \cdots, i_d)\in \Z^d$ by defining
\[
x[I] = x[i_1\modd n_1, \cdots, i_d\modd n_d].
\]

\begin{definition}[Translation]
A translation $T^M:\H \to \H$ with $M\in \Z^d$
is an invertible linear operator
such that for all $I\in \Z^d$ and $x\in \H$,
\[
T^M(x)[I] = x[I-M].
\]
The inverse of $T^M$ is clearly  $T^{-M}$.
\end{definition}
\begin{definition}[Equivariance]
A map $w:\H \to \H$ is called {\em equivariant with respect to translations}
if
for all $x\in \H$ and $M\in\Z^d$,
\[
T^M(w(x))=w(T^M(x)).
\]
\end{definition}

\begin{definition}[Vectorization]
 A tensor $x$ can be vectorized to
 $X\in \overrightarrow{\H} =  \mathbb{D}^N$
 with $N= n_1 n_2 \cdots n_d$ such that
\begin{align*}
X(\delta(I)) \coloneqq   x[I] ,
\end{align*}
     where $\delta(I) \coloneqq (i_1 \modd n_1)n_2 n_3\cdots n_{d}+(i_2 \modd n_2)n_3 n_4\cdots n_{d}+\cdots +(i_d \modd n_d)$, and
 we denote $X = \overrightarrow{x}$.
Moreover, the translation $T^M$ is vectorized as
$T^M(X)\coloneqq \overrightarrow{T^M(x)}$.
\end{definition}

\subsection{Equivariant operators}\label{subsec:equi-operators}
When $\D = \R$, the tensor space $\H$ is a Hilbert space
by defining the inner product as $x\cdot z \coloneqq \overrightarrow{x}\cdot \overrightarrow{z}$
which is the inner product in vector space $\overrightarrow{\H}$.
In the rest of this section,
we assume $\D = \R$.

According to the Reize's representation theorem,
there is a bijection between continuous linear operator space and tensor space.
That is, a continuous linear operator $v: \H\to \R$ can be viewed as a tensor $v\in \H$
satisfying   $v(x)=v \cdot x$.
Now we can translate $v$ by $T^M$ and obtain $T^M(v): \H\to \R$ such that $T^M(v)(x)=T^M(v)\cdot x$.

We consider a continuous linear operator $w:\H \to \H$.
For $I\in\Z^d$ and $x\in\H$, denote $w_{I}(x) = w(x)[I]$.
Then $w_{I}:\H\to \R$ is a continuous linear operator.
%
An {\em affine operator} $\alpha:\H \to \H$ differs from a continuous linear operator $w$
by a {\em bias tensor} $c$ such that
$\alpha(x) = w(x) + c$ for all $x \in \H$.

\begin{restatable}{theorem}{thmequi}\label{thm:equi}
Let $\alpha(x)=w(x)+c:\H\to \H$ be an affine operator.
Then, $\alpha$ is  equivariant with respect to translations
if and only if
for all $M\in\Z^d$,
\begin{align*}
w_{M}  = T^M(w_{\zero}) \text{  and  }
c \propto \one,
\end{align*}
where  $\zero$ is the zero vector in $\Z^d$ and
$ c \propto \one$ means that $c$ is a {\em constant tensor}, that is, all of its entries are the same.
\end{restatable}

Proof of Theorem \ref{thm:equi}
is given in Appendix \ref{app:equi}.
Recall that $\overrightarrow{\H}=\R^N$ is the vectorization of $\H$ and
$T^M$ also translates vectors in $\overrightarrow{H}$.
Each continuous linear operator on $\H$ corresponds to
a matrix in $\R^{N\times N}$ and each bias operator corresponds to a vector in $\R^N$.
Now we consider the translation equivariance in vector space.
\begin{definition}[Circular filter] \label{def:circular-filter}
Let $W=(W_0,W_1,\cdots ,W_{N-1})^T$ be a matrix in $\R^{N\times N}$.
$W$ is called a \emph{circular filter} if
$W_{\delta(M)} =  T^M(W_0)$ for all $M\in\Z^d$.
\end{definition}
As the vector version of Theorem \ref{thm:equi}, we have
%
\begin{corollary}
\label{cor:equi}
Let $A: \R^N \to \R^N$ be an affine transformation such that
\[
A(X) = W \cdot X + C,
\]
in which $W\in \R^{N\times N}$, $C\in \R^N$.
Then, $A$ is equivariant with respect to translations in the sense that
for all $M\in \Z ^d$
\[
A(T^M(X)) = T^M(A(X))
\]
if and only if
$W$ is a circular filter
and $C\propto \one$.
\end{corollary}

This affine transformation is very similar to the commonly used
convolutional layers~\cite{fukushima1982neocognitron, lecun1998gradient}
in terms of shared parameters and the alike convolutional operation.
But the strict equivariance calls for
the same in-size and out-size, and circular convolutions,  
which are usually violated by CNNs.

\subsection{Equivariant neural networks} \label{subsec:equi-network}
To compose a strictly translation-equivariant network,
the spatial sizes of the input and output in each layer must be the same
and thus down-samplings are not allowed.
Though  Corollary~\ref{cor:equi} provides the fundamental component
of a strictly translation-equivariant network,
different compositions of this component lead to various
equivariant networks.
Here we give the \emph{canonical architecture}.
We construct the strictly translation-equivariant network $F$ with $L$ layers as
\begin{align}
    F(X) = F_L\circ F_{L-1}\circ \cdots \circ F_1(X).\label{equ:network}
\end{align}
The $l$-the layer $F_l$ has $n_l$ channels and
for an input $X\in \R^{n_{l-1}\times N}$ we have
\begin{align}
    F_l(X) = \sigma(W[l]\cdot X + C[l])\in \R^{n_l \times N},\label{equ:layer}
\end{align}
where
\begin{align*}
    W[l] &= (W^1[l], \cdots, W^{n_l}[l])\in \R^{n_l\times n_{l-1}\times N\times N},\\
    C[l] &= (C^1[l]\cdot \one, \cdots,C^{n_l}[l]\cdot \one),\\
    W^k[l] &=(W^{k,1}[l],\cdots,W^{k,n_{l-1}}[l])\in \R^{n_{l-1}\times N\times N},\\
    C^k[l] &=(C^{k,1}[l],\cdots,C^{k, n_{l-1}}[l]) \in \R^{n_{l-1}},
\end{align*}
$\sigma$ is the activation,
$W^{k, r}[l]\in \R^{N\times N}$ are circular filters, $C^{k, r}[l]\in \R$ are constant biases
for $k=1,\cdots,n_l$ and $r=1,\cdots,n_{l-1}$,
the $\cdot$ denotes the inner product and $\one$ is the vector whose all components are 1.

\section{Translation restorer}\label{sec:restorer}

\subsection{Method}\label{subsec:method}

In Section~\ref{subsec:equi-network},
we propose a strictly equivariant neural network architecture (\ref{equ:network})
such that any translation on the input will be reflected on the output.
Generally speaking,
once the outputs of an equivariant network on a dataset
have some spatial structure,
this structure shifts consistently as the input shifts.
Thus, the translation parameter of a shifted input can be solved from its output.
Finally, we can restore the input via the inverse translation.
Figure~\ref{fig:arch} shows how a restorer works.
\begin{figure}[ht]
\centering
\includegraphics[scale=0.5]{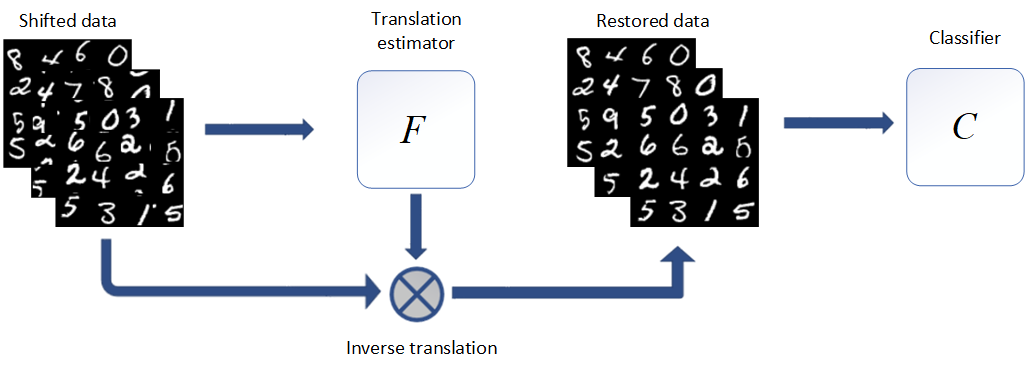}
\caption{
The pre-classifier translation restorer.
For a shifted data $T^M(x)$ as the input, the translation estimator
obtains the translation $M$
and restore the original data $T^{-M}(T^M(x))=x$, which is feed into a pre-trained classifier.
}
\label{fig:arch}
\end{figure}

The whole restoration process splits into two stages,
translation estimation and inverse translation.
We first define the translation estimator
which outputs a consistent and special structure on a dataset.

\begin{definition}\label{def:estimator}
Let $\mathcal{D}\subset \D^{P\times N}$ be a
dataset with $P$ channels.
Then a translation-equivariant network
\begin{align*}
F: \R^{P \times N} \to \R^{N}
\end{align*}
is said to be a translation estimator for $\mathcal{D}$ if
\begin{align*}
F(X)[0] = \mymax_{i=0} ^{N-1} F(X)[i],
\end{align*}
where $F(X)[i]$ is the $i$-th component of $F(X)$.
\end{definition}

Given such a translation estimator for dataset $\mathcal{D}$
and a shifted input $X' = T^M(X)$ for some $X\in \mathcal{D}$,
we propagate $X'$ through $F$ and get the output $F(X')\in \R^N$.
Since the first component of $F(X)$ is the largest,
the location of the largest component of $F(X')$
is exactly the translation parameter:
\begin{align*}
\delta(M) = \myargmax_{i=0}^{N-1} F_i(X').
\end{align*}
Then we can restore $X=T^{-M}(X')$ by inverse translation.
The restored inputs can be feed to any classifier trained on the dataset $\mathcal{D}$.

\subsection{Existence of the restorer}
In  this section, we show the existence of restorers,
that is, the translation estimator.
Note that our restorer is independent of the following classifier
but dependent on the dataset.
For translation, if a dataset contains both an image and a translated version of it,
the estimator must be confused.
We introduce aperiodic datasets to clarify such cases.
\begin{definition}[Aperiodic dataset]\label{def:aperiodic}
    Let $\mathcal{D}\subset \D^{P\times N}$ be a finite dataset with $P$ channels.
    We call $\mathcal{D}$ an aperiodic dataset if
    $\zero\notin \mathcal{D}$ and
    \begin{align*}
        T^{M}(X) = X' \iff M=\zero \hbox{ and } X=X',
    \end{align*}
    for $M\in \Z^{d+1}$ and $X, X'\in \mathcal{D}$.
    Here $d$ is the spatial dimension and $M$ decides the translation in the channel dimension in addition.
\end{definition}

Let $\mathcal{D}$ be an aperiodic dataset.
Given that $\D=[2^{Q+1}]$ which is the case in image classification,
we prove the existence of the translation estimator for such an aperiodic dataset.
The proof consists of two steps.
The data are first mapped to their binary decompositions through a
translation-equivariant network as Equation~\eqref{equ:network}
and then the existence of the translation-restorer
in the form of Equation~\eqref{equ:network} is proved for binary data.

Let $\D=[2^{Q+1}]$ and $\B=\{0, 1 \}$.
We denote $\eta:\D \to \B^{Q}$
to be the binary decomposition,
such as $\eta(2)=(0,1,0)$ and $\eta(3)=(1,0,1)$.
We perform the binary decomposition on  $X \in \D^{P\times N}$ element-wisely
and obtain $\eta(X)\in \B^{G\times N}$,
where $G=PQ$ is the number of channels in binary representation.
A dataset $\mathcal{D}\subseteq \D^{P \times N}$
can be decomposed into
$\mathcal{B}\subset \B^{G\times N}$.
Note that the dataset $\mathcal{D}$ is aperiodic if and only if its
binary decomposition $\mathcal{B}$ is aperiodic.

The following Lemma~\ref{lem:approx} demonstrates
the existence of a translation-equivariant network
which coincides with the binary decomposition $\eta$
on $[2^{Q +1}]^{P\times N}$.
Proof details are  placed in Appendix~\ref{app:approx}.
\begin{restatable}{lemma}{lemapprox}\label{lem:approx}
    Let $\eta: [2^{Q+1}]\to \B$ be the binary decomposition.
    There exists a $(2Q+2)$-layer network $F$
    in the form of Equation~\eqref{equ:network} with ReLU activations
    and width at most $(Q+1)N$
    such that for $X\in [2^{Q+1}]^{P\times N}$
    \begin{align*}
        F(X) = \eta(X).
    \end{align*}
\end{restatable}

The following Lemma~\ref{lem:binary-exist} demonstrate
the existence of a 2-layer translation restorer for an aperiodic binary dataset.
Proof details are placed in Appendix~\ref{app:binary-exist}.
\begin{restatable}{lemma}{lemexistence}\label{lem:binary-exist}
Let $\mathcal{B} = \{Z_s|s=1,2,\cdots, S \} \subset \B^{G\times N}$ be an aperiodic binary dataset.
Then there exists a 2-layer network $F$
in the form of Equation~\eqref{equ:network} with ReLU activations and width at most $SN$
such that for all $s=1,2,\cdots ,S$,
\begin{align*}
F(Z_s)[0] = \mymax_{i=0}^{N-1}F(Z_s)[i].
\end{align*}
\end{restatable}

Given a $(2Q+2)$-layer network $F'$ obtained from Lemma~\ref{lem:approx}
and a 2-layer network $F''$ obtained from Lemma~\ref{lem:binary-exist},
we stack them and have $F=F''\circ F'$ which is exactly a translation restorer. We thus have proved the following theorem.

\begin{restatable}{theorem}{thmmain} \label{thm:main}
Let $\mathcal{D} = \{X_s|s=1,2,\cdots, S \} \subset [2^{Q+1}]^{P\times N}$ be an aperiodic dataset.
Then there exists a network $F:\R^{P \times N} \to \R^{N}$
in the form of Equation~\eqref{equ:network} with ReLU activations
such that for $s=1,2,\cdots ,S$,
\begin{align*}
F(X_s)[0] = \mymax_{i=0}^{N-1}F(X_s)[i],
\end{align*}
of which the depth is at most $2Q+4$ and the width is at most $\max (SN, (Q+1)N)$.
Namely, this network is a translation restorer for $\mathcal{D}$.
\end{restatable}

\begin{figure}[htbp]
\centering
    \begin{minipage}{0.3\textwidth}
    \centering
    \includegraphics[width=1\textwidth]{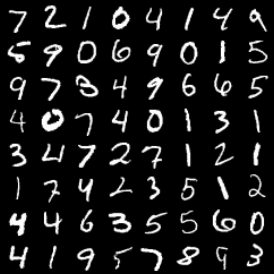}
    \end{minipage}
    \begin{minipage}{0.3\textwidth}
    \centering
    \includegraphics[width=1\textwidth]{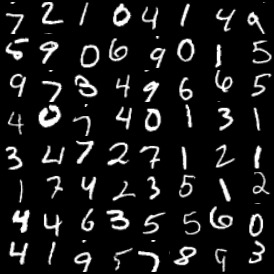}
    \end{minipage}
    \begin{minipage}{0.3\textwidth}
    \centering
    \includegraphics[width=1\textwidth]{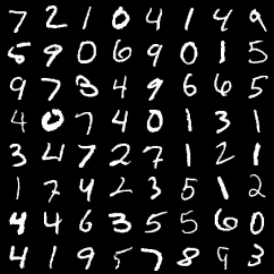}
    \end{minipage}
    \begin{minipage}{0.3\textwidth}
    \centering
    \includegraphics[width=1\textwidth]{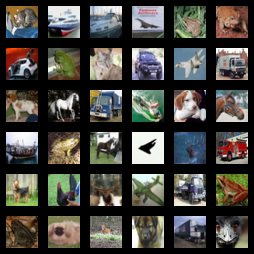}
    \centerline{Original}
    \end{minipage}
    \begin{minipage}{0.3\textwidth}
    \centering
    \includegraphics[width=1\textwidth]{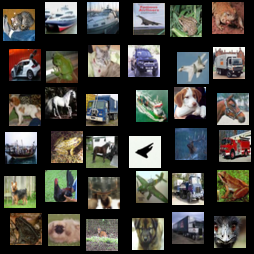}
    \centerline{Shifted}
    \end{minipage}
    \begin{minipage}{0.3\textwidth}
    \centering
    \includegraphics[width=1\textwidth]{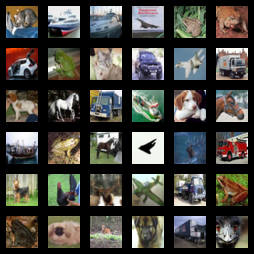}
    \centerline{Restored}
    \end{minipage}
\caption{The restorers for MNIST and CIFAR-10.}
\label{fig:compare}
\end{figure}

\section{Experiments} \label{sec:exp}
The core component of the restorer is the translation estimator which outputs the translation parameter of the shifted inputs.


We use the architecture described in Equation~\eqref{equ:network}
with $L=6$, $n_l=1$ for $l=1,\cdots,L$ and ReLU activations.
The training procedure aims at maximizing the first component of the outputs.
Thus the max component of the output indicates the input shift.
The experimental settings
are given in Appendix \ref{app:exp}.
We report four sets of experiments below.

\subsection{Translation restoration}
 We first focus on the performance of translation restoration.
 Experiments are conducted on MNIST, CIFAR-10, and 3D-MNIST.

\paragraph{\bf 2D Images.}
We train translation restorers for MNIST and CIFAR-10.
MNIST images are resized to 32x32 and CIFAR-10 images are padded 4 blank pixels at the edges.

In Figure~\ref{fig:compare},
the left column is the original images,
the middle column is the randomly shifted images
and the right column is the restored images.
On both datasets,
images are randomly shifted vertically and horizontally
at most $\frac{1}{4}$ of its size.
The shift is a circular shift where pixels shifted out of the figure appear on the other end.
We can see that the shifted images are disorganized
but the restored images are very alike the original images.

To evaluate the restoration performance of pretrained restorers, we train  classifiers and test them on randomly shifted images and restored ones and the results are given in Table \ref{tab-11}.
When images are not shifted, the restorers lead to only $0.3\%$ and $0.03\%$ accuracy reduction on two datasets.
Nevertheless, even if the translation scope is 1, restorers improve the accuracy.
Moreover, no matter how the images are shifted, the restorer can repair them to the same status and result in the same classification accuracy, namely $98.58\%$ and $88.18\%$,
while the accuracies drop significantly without the restorer, and the larger the range of translation, the more obvious the restoration effect

\begin{table}[htb]
\centering
\caption{Restoration performance on MNIST and CIFAR-10. Images are randomly shifted within the translation scope ranging from 0 to 8 in both vertical and horizontal directions. We use LeNet-5 on MNIST and ResNet-18 on CIFAR-10. "Res." and "Trans." stand for restorer and translation respectively.}
\label{tab-11}
\resizebox{1\textwidth}{!}{
\begin{tabular}{ccccccccccc}
\toprule
                          & Res.\textbackslash{}Trans. & 0     & 1     & 2     & 3     & 4     & 5     & 6     & 7     & 8     \\
\midrule
\multirow{3}{*}{MNIST}    & w/o                        & 98.89 & 98.21 & 95.41 & 87.07 & 76.61 & 62.9  & 51.33 & 41.1  & 35.7  \\
                          & w/                         & 98.59 & 98.59 & 98.59 & 98.59 & 98.59 & 98.59 & 98.59 & 98.59 & 98.59 \\
                          & Effect                     & -0.3  & +0.38  & +3.18  & +11.52 & +21.98 & +35.69 & +47.26 & +57.49 & +62.89 \\
\midrule
\multirow{3}{*}{CIFAR-10} & w/o                        & 88.21 & 86.58 & 85.9  & 83.65 & 82.16 & 80.46 & 79.37 & 77.71 & 76.01 \\
                          & w/                         & 88.18 & 88.18 & 88.18 & 88.18 & 88.18 & 88.18 & 88.18 & 88.18 & 88.18 \\
                          & Effect                     & -0.03 & +1.6   & +2.28  & +4.53  & +6.02  & +7.72  & +8.81  & +10.47 & +12.17 \\
\bottomrule
\end{tabular}
}
\end{table}

\paragraph{\bf Different Architectures}
Our proposed restorer is an independent module that can be placed before any classifier.
It is scalable to different architectures the subsequent classifier uses.

In Table~\ref{tab:architecture}, we evaluate the restoration performance on popular architectures including VGG-16, ResNet-18, DenseNet-121, and MobileNet v2.
Translated mages (w/Trans.) are randomly shifted within scope 4 in both vertical and horizontal directions.
The reduction of accuracy on original images is no more than $0.04\%$ and the restorer improves the accuracy on shifted images by $1.66\%\sim 6.02\%$.

\begin{table}[htb]
\centering
\caption{Restoration performance on different architectures and CIFAR-10.}
\label{tab:architecture}
\begin{tabular}{cccccccccccc}
\toprule
                          & \multicolumn{2}{c}{VGG-16} &  & \multicolumn{2}{c}{ResNet-18} &  & \multicolumn{2}{c}{DenseNet-121} &  & \multicolumn{2}{c}{MobileNet v2} \\
Res.\textbackslash Trans. & w/o        & w/         &  & w/o          & w/          &  & w/o           & w/           &  & w/o           & w/            \\
\midrule
w/o                       & 89.27      & 83.40      &  & 88.21        & 82.16       &  & 92.14         & 90.46        &  & 88.10         & 83.36         \\
w/                        & 89.23      & 89.23      &  & 88.18        & 88.18       &  & 92.12         & 92.12        &  & 88.09         & 88.09         \\
Effect                    & -0.04      & +5.83      &  & -0.03        & +6.02       &  & -0.02         & +1.66        &  & -0.01         & +4.73         \\
\bottomrule
\end{tabular}
\end{table}

\paragraph{\bf Translation Augmentation.}
Training with translation augmentation is another approach to improving the translational invariance of a model.
However, translation augmentation is limited in a certain scope and thus cannot ensure the effectiveness for test images shifted out of the scope.

In Figure~\ref{fig:aug-res}, we compare the restoration performance on models not trained with translation augmentation (dash lines) and models trained with translation augmentation (solid lines).
The augmentation scope is $10\%$ of the image size, that is, 3 pixels for MNIST and 4 pixels for CIFAR-10.
Translation augmentation indeed improves the translational invariance of the classifier on images shifted in the augmentation scope.
However, when the shift is beyond the augmentation scope, the accuracy begins to degrade.
In such a case, the pre-classifier restorer is still able to calibrate the shift and improve the accuracy of the classifier trained with augmentation.

\begin{figure}[h]
\centering
    \begin{minipage}{0.49\textwidth}
    \centering
    \includegraphics[width=0.95\textwidth]{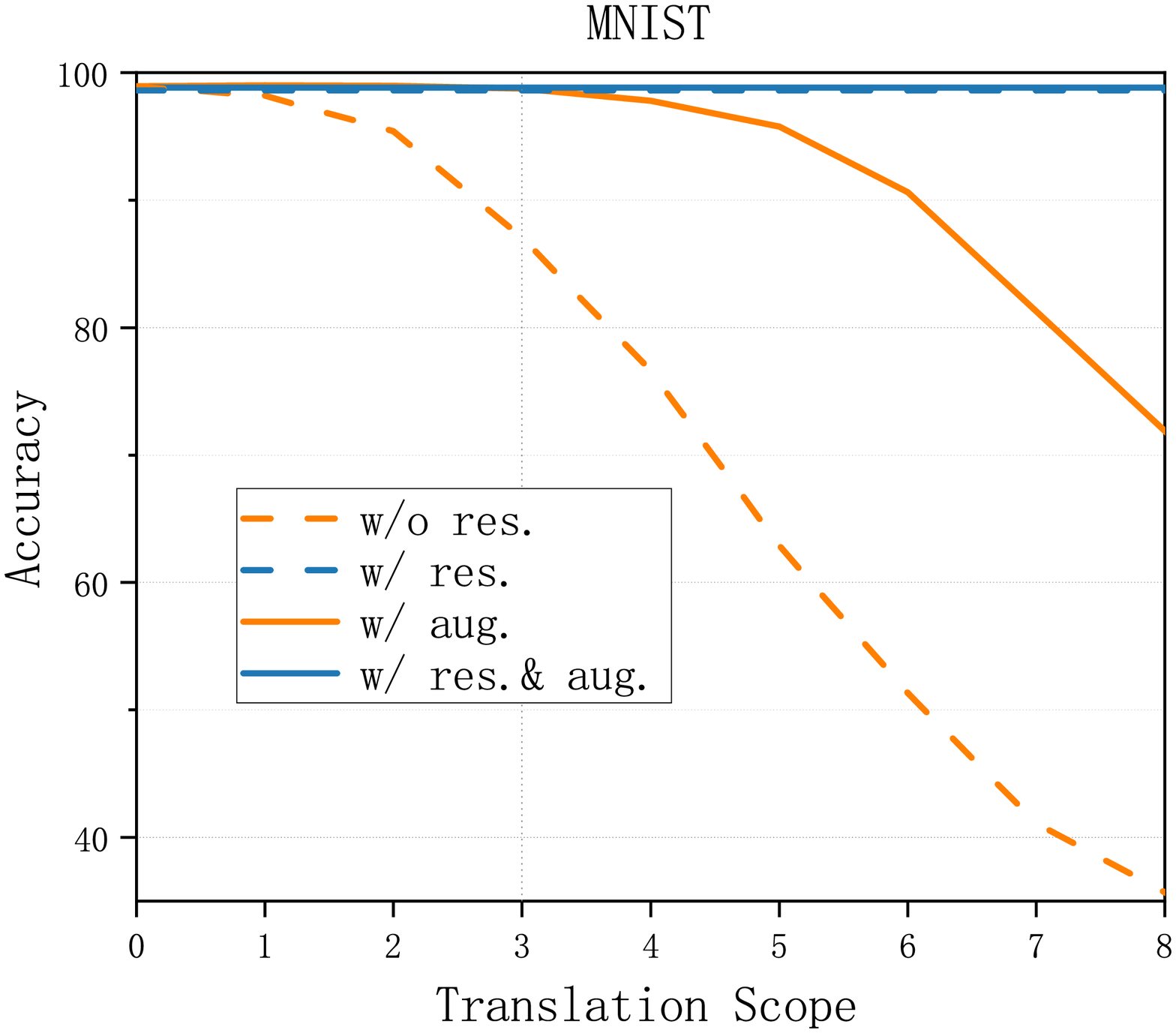}
    \end{minipage}
    \begin{minipage}{0.49\textwidth}
    \centering
    \includegraphics[width=0.95\textwidth]{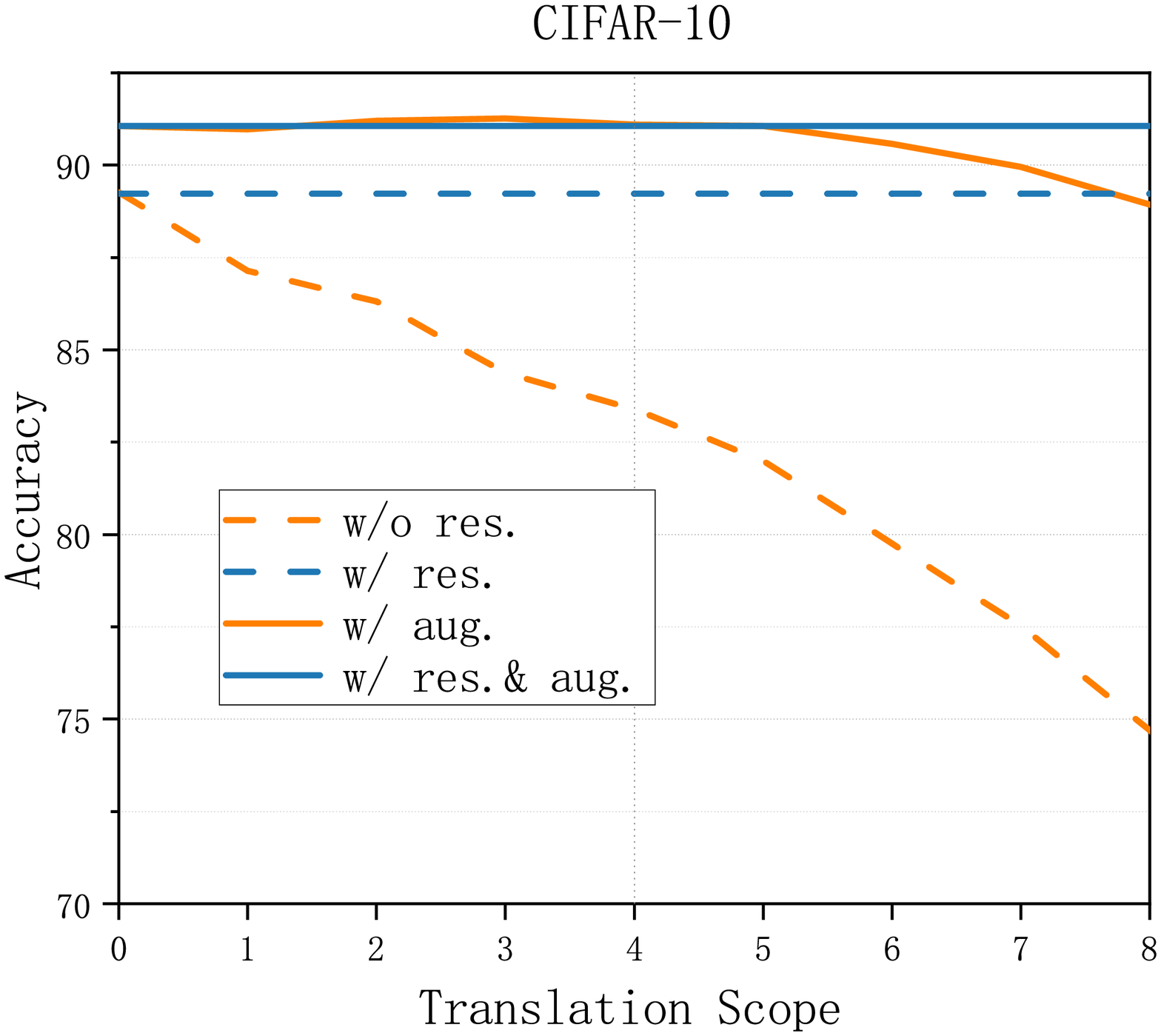}
    \end{minipage}
\caption{Restoration performance on classifiers trained with or without translation augmentations. The models are LeNet-5 for MNIST and VGG-16 for CIFAR-10.
"res." and "aug" stand for restorer and augmentation, respectively.}
\label{fig:aug-res}
\end{figure}

\paragraph{\bf 3D Voxelization Images.}
3D-MNIST contains 3D point clouds generated from images of MNIST.
The voxelization of the point clouds contains grayscale 3D tensors.

Figure~\ref{fig:compare3d} visualizes the restoration on 3D-MNIST.
In the middle of each subfigure,
the 3-dimensional digit is shifted in a fixed direction.
This fixed direction is detected by the translation estimator
and the restored digit is shown on the right.

\begin{figure}[h]
\centering
\includegraphics[width=5cm]{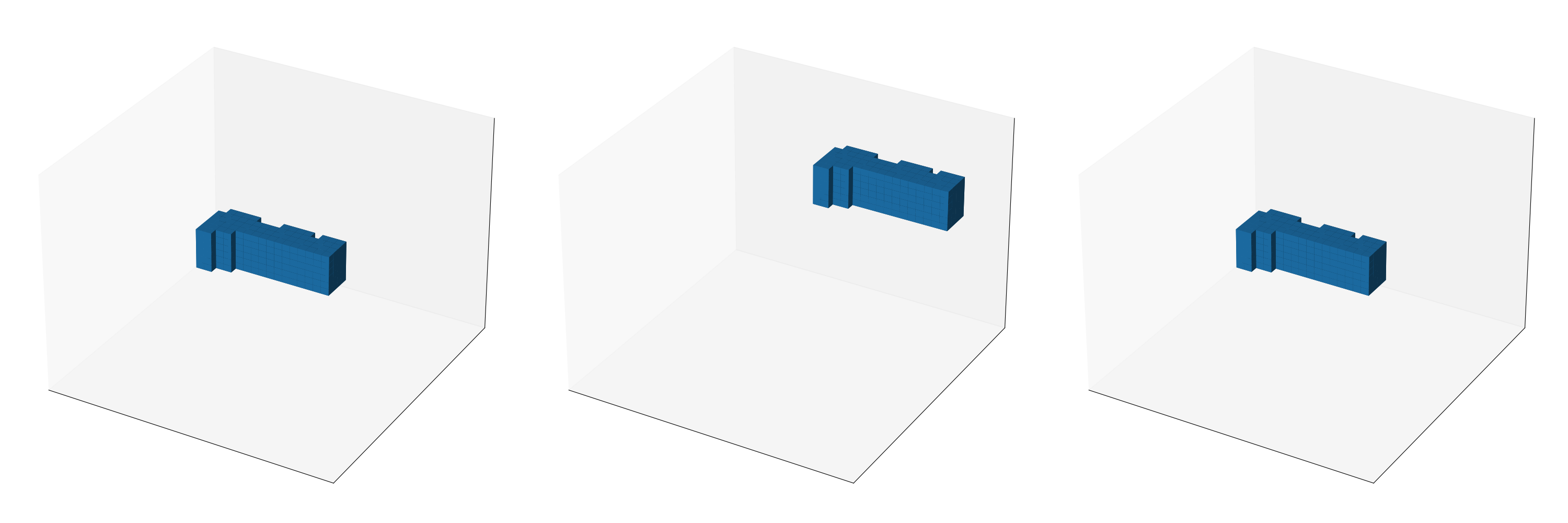}
\quad
\quad
\includegraphics[width=5cm]{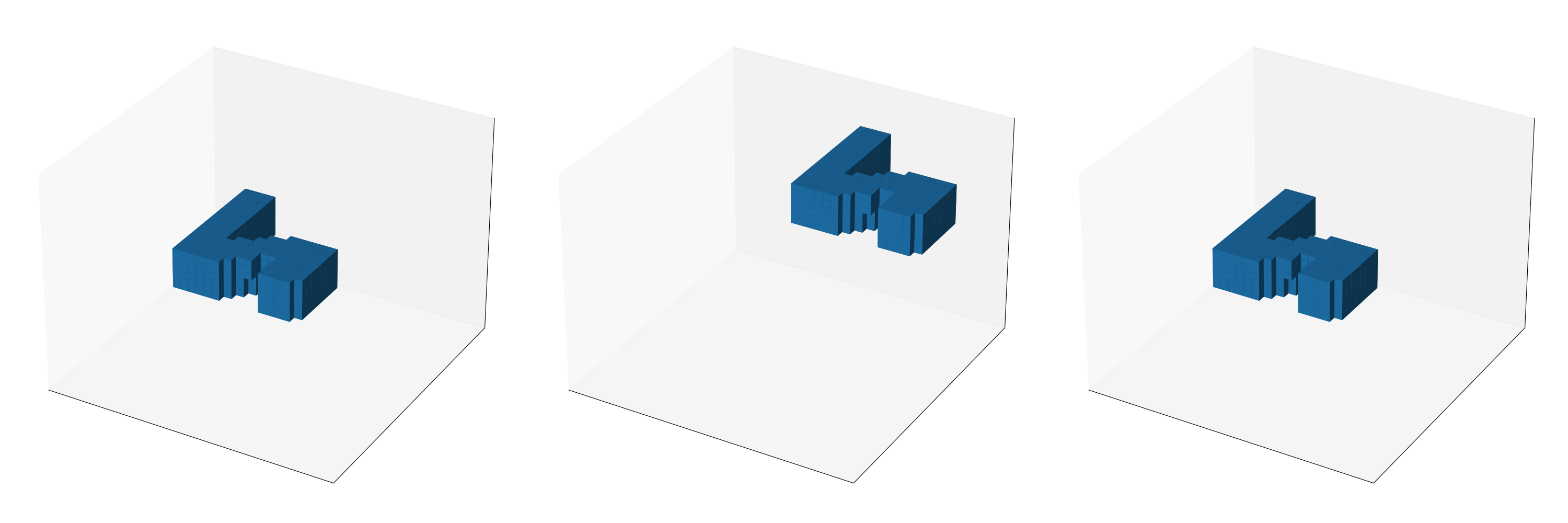}
\quad
\includegraphics[width=5cm]{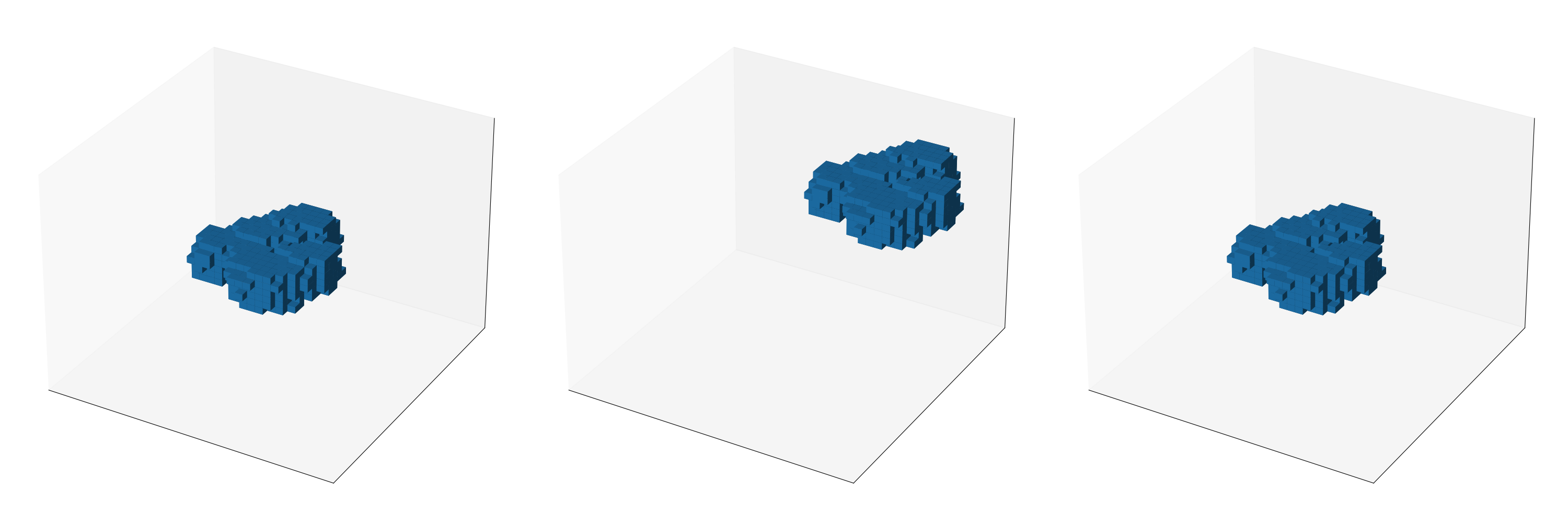}\\
\includegraphics[width=5cm]{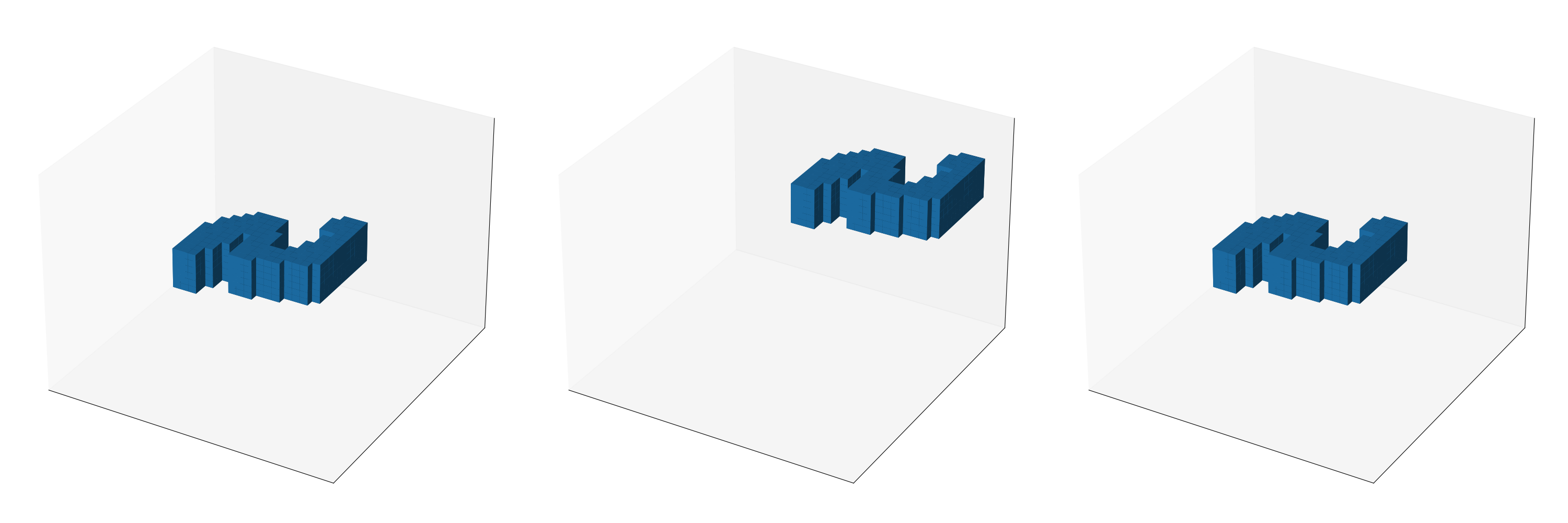}
\quad
\includegraphics[width=5cm]{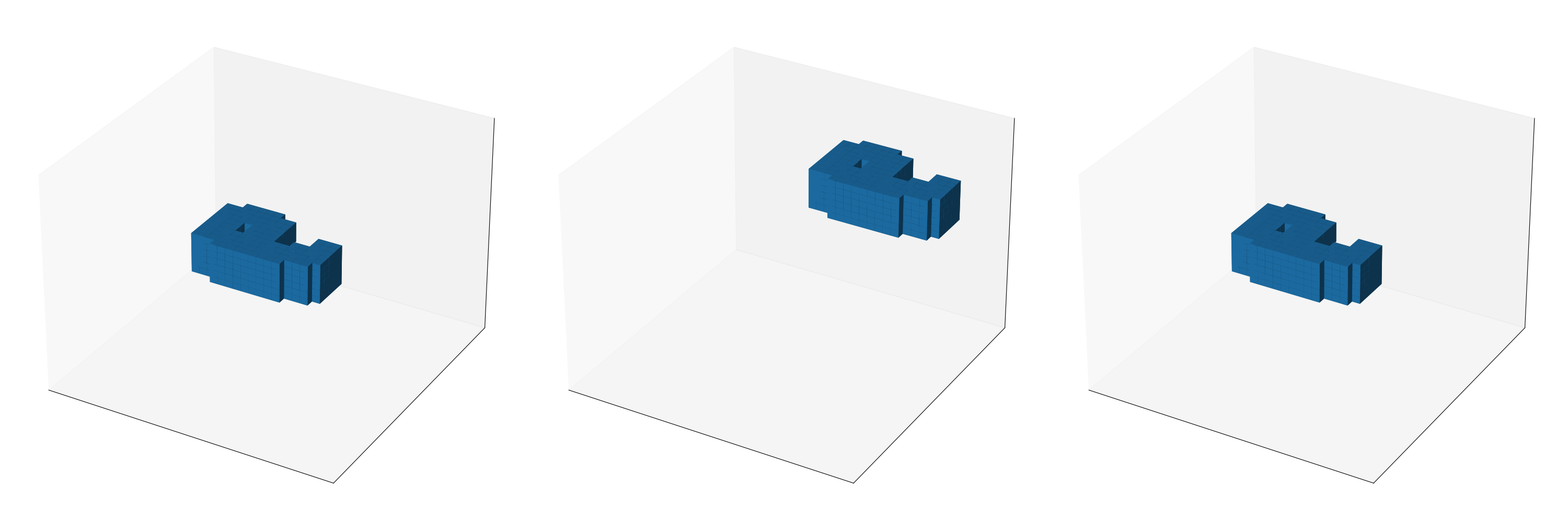}
\quad
\includegraphics[width=5cm]{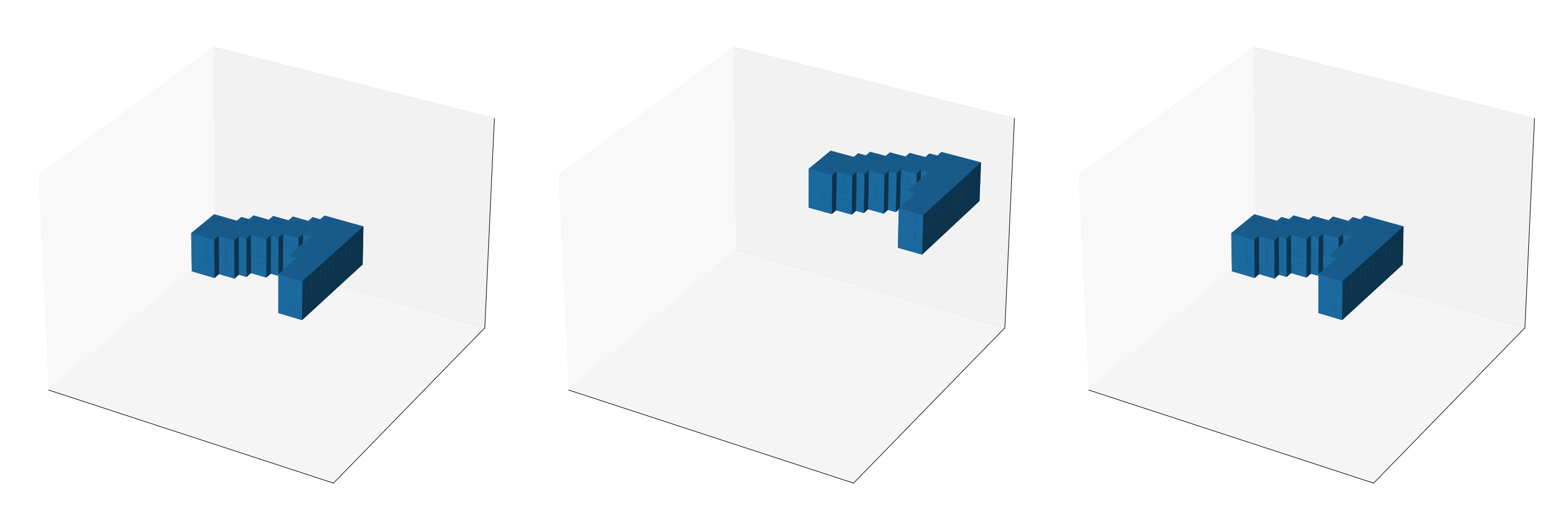}
\caption{
The restorer on 3D-MNIST.
In each sub-figure, the left is the original digit,
the middle is shifted digit, and the right is the restored digit.
}
\label{fig:compare3d}
\end{figure}

 \subsection{Rotation restoration}
 Rotation can be regarded as a kind of translation.
 The   Euclidean space $\R^{d+1}$
 can be characterized by polar coordinates
 \begin{align*}
     (\phi_1,\cdots,\phi_{d-1},\phi_{d}, r) \in [0,\pi] \times \cdots \times [0,\pi] \times [0,2\pi) \times \R^+.
 \end{align*}
 We can sample a property map, defined in Section~\ref{subsec:equi-tensor},
 $\tilde{x}:\R^{d+1}\to \R$ over a $(n_1, n_2, \cdots, n_{d+1})$-grid
 along the polar axes and obtain a tensor $x$ such that for given $R>0$ and $0<a<1$
 \begin{align*}
     x(I) = \tilde{x}(\frac{\pi i_1}{n_1-1}, \cdots, \frac{\pi i_{d-1}}{n_{d-1}-1}, \frac{2\pi i_{d}}{n_{d}},  Ra^{i_{d+1}}  ),
 \end{align*}
 where $I=(i_1,i_2,\cdots,i_{d+1})\in \prod_{i=1}^{d+1}[n_i]$.
 The last index $Ra^{i_{d+1}}$ can be replaced with $\frac{i_{d+1} R}{n_{d+1}}$.
 Note that the vectorization $X$ is in  $\D^{n_{d+1}\times N}$ with $N=n_1 n_2 \cdots n_d$.
 Since we only care about the rotation, i.e.\
 the translation $T^M$ on the first $d$ components of $I$,
 the space $\D^{n_{d+1}\times N}$ is viewed as an $n_{d+1}$-channel vector space.
 Thus, we can treat rotation as translation
 and leverage the method discussed above to restore rotated inputs.

 \paragraph{Visualization.}
 We experiment with the rotation restoration on MNIST.
 Since the transfer from Cartesian coordinates to polar coordinates requires high resolution,
 we first resize the images to $224\times 224$ pixels.
 Other settings are similar to the aforementioned experiments.

 Figure~\ref{fig:rot} visualizes the rotation restoration.
 We can tell from it that most rotated images are restored correctly,
 though some of them are not restored to the original images.
 {The rotated digit 9 in the top row is more like an erect digit 6 than the original one
 and the restorer just leaves it alone.}
 The reason why rotation restoration is not as perfect as translation restoration
 is that the dataset is not aperiodic with respect to rotations.
 On the one hand, some digits seem like the rotated version of other digits,
 such as 6 and 9.
 On the other hand, even in a certain digit class, images vary in digit poses.
 {For example, a rotated digit 0 is similar to an erect one.}
 \begin{figure}[htbp]
 \centering
 \includegraphics[scale=0.23]{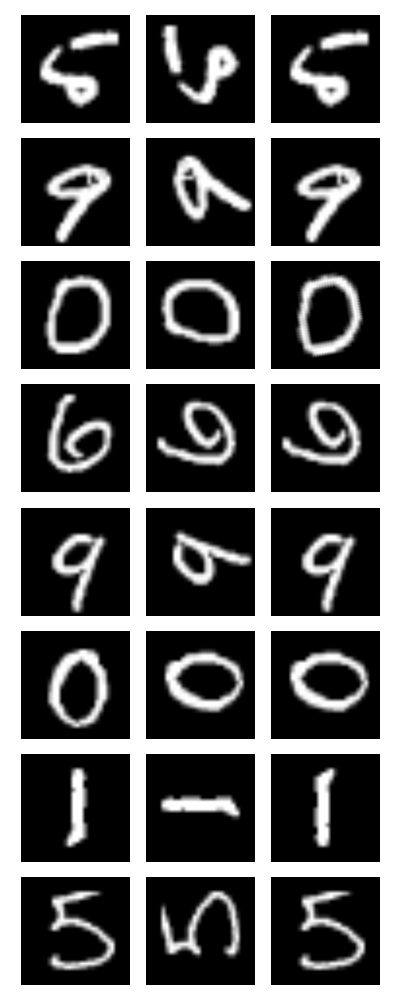}
 \includegraphics[scale=0.23]{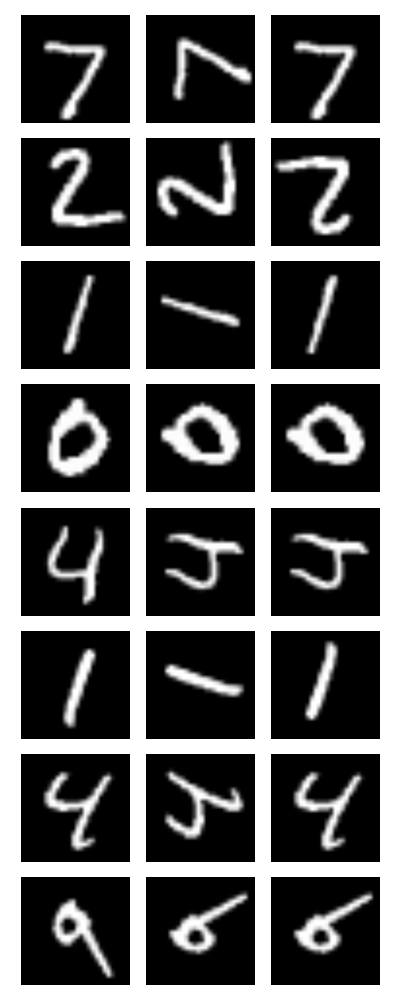}
 \includegraphics[scale=0.23]{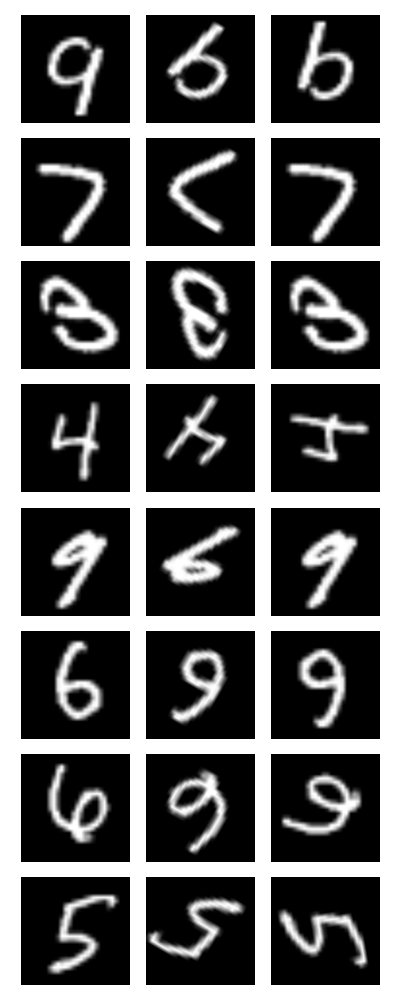}
 \caption{ The rotation restoration on the first 24 images in the test set of MNIST.
 The left column of each subfigure is the original images and the right is the restored images.
 In the middle column of each subfigure,
 the images are rotated $40^\circ, 90^\circ$ and $150^\circ$ respectively. }
 \label{fig:rot}
 \end{figure}
 
 Note that the group-equivariant CNNs in~\cite{cohen2016group,gens2014deep}
 can only deal with rotations by certain angles such as $90^\circ$, $180^\circ$,
 and $270^\circ$.
 On the other hand, our approach can deal with rotations from more angles.

\section{Conclusion}

This paper contributes to the equivalent neural networks in two aspects.
Theoretically, we give the sufficient and necessary conditions for an
affine operator $Wx+b$ to be  translational equivariant,
that is, $Wx+b$ is translational equivariant on a tensor space if and only if $W$ has the high dimensional convolution structure and $b$ is a constant tensor.
It is well known that if $W$ has the convolution structure, then  $Wx$ is equivariant to translations~\cite{fukushima1982neocognitron,he2016deep},
and this is one of the basic principles behind the design of CNNs.
Our work gives new insights into the convolutional structure used in CNNs
in that, the convolution structure is also the necessary condition and hence the most general structure for translational equivariance.
Practically, we propose the translational restorer to recover the original images from the translated or rotated ones.
The restorer can be combined with any classifier to alleviate
the performance reduction problem for translated or rotated images.
As a limitation, training a restorer on a large dataset such as the ImageNet is still computationally difficult.


\newpage

\appendix

\section{Proof of Theorem \ref{thm:equi}}
\label{app:equi}
We first prove a lemma.
\begin{lemma} \label{lem:trans}
Let $v:H\to \R$ be a continuous linear operator.
We have
\[
v(T^M(x)) = T^{-M}(v)(x),
\]
for all $x\in H$ and all $M\in\Z^d$.
\end{lemma}
\begin{proof}
A continuous linear operator $v$ can be viewed as a tensor in $H$.
We have
\begin{align*}
v(T^M(x))&= v\cdot T^M(x)   \\
&=\sum_{I\in \prod_{i=1}^d [n_i]} v[I]\cdot T^M(x)[I]\\
&=\sum_{I\in \prod_{i=1}^d [n_i]} v[I]\cdot x[I-M]\\
&=\sum_{I\in \prod_{i=1}^d [n_i]} v[I+M]\cdot x[I]\\
&=\sum_{I\in \prod_{i=1}^d [n_i]} T^{-M}(v)[I]\cdot x[I]\\
&=T^{-M}(v)\cdot x\\
&=T^{-M}(v)(x).
\end{align*}
\end{proof}

\thmequi*
\begin{proof}
Here we denote the index by $I=(i_1, i_2, \cdots, i_d)$.
 On the one hand,
\begin{align*}
&T^M(\alpha(x))[I]\\
=&T^M(w(x))[I] + T^M(c)[I]\\
=&w(x)[I-M]+c[I-M]\\
=&w_{I-M}(x) + c[I-M].
\end{align*}
On the other hand,
\begin{align*}
&\alpha(T^M(x))[I]\\
=&w(T^M(x))[I]+ c[I]\\
=&w_{I}(T^M(x)) + c[I]\\
=&T^{-M}(w_{I})(x) + c[I],
\end{align*}
in which the last equation is from Lemma~\ref{lem:trans}.

\textbf{Sufficiency}
Assume for all $M\in \Z^d$,
\begin{align*}
w_{M} = T^M(w_\zero) \text{  and  }
c \propto \one.
\end{align*}
We have
\begin{align*}
T^{-M}(w_{I}) =&T^{-M}(T^I(w_{\zero})) \\
=&T^{I-M}(w_\zero)\\
=&w_{I-M},\\
c[I-M] =&c[\zero]\\
=& c[I],\\
 T^M(\alpha(x))[I] =& \alpha(T^M(x))[I].
\end{align*}
Thus,
\begin{align*}
T^M(\alpha(x)) =& \alpha(T^M(x)).
\end{align*}

\textbf{Necessity}
Assume $\alpha$ is equivariant with respect to translations in the sense that
   \begin{align*}
 T^M(\alpha(x)) =& \alpha(T^M(x)).
\end{align*}
We have
\begin{align*}
   w_{I-M}(x) - T^{-M}(w_{I})(x) = c[I] - c[I-M].
\end{align*}
Fix indices $I=\zero$ and obtain that for all $M\in \Z^d$,
\begin{align*}
   w_{M}(x) - T^{M}(w_{\bm{0}})(x) = c(\zero) - c(M).
\end{align*}
Recall that a continuous linear operator can be viewed as a tensor
where the operation is the inner product in tensor space.
We have
   \begin{align*}
   c(\zero) - c(M) &= (w_{M} - T^{M}(w_{\zero}) )\cdot x
   = \overrightarrow{w_{M} - T^{M}(w_{\zero})}\cdot \overrightarrow{x}.
\end{align*}
For each fixed $M$,
the left side is a constant scalar
    and the right side is a linear transformation on all vector $\overrightarrow{x}\in \overrightarrow{H}$.
Thus, the both sides are equal to zero tensor and we have
\begin{align*}
   c(\zero) &= c(M), \\
   \overrightarrow{w_{M}} &= \overrightarrow{ T^{M}(w_{\zero})}.\\
\end{align*}
That is, for all $M\in \Z^d$,
 \begin{align*}
   c &\propto \one , \\
   w_{M} &= T^{M}(w_{\zero}).
\end{align*}
\end{proof}


\section{Proof of Lemma~\ref{lem:approx}}\label{app:approx}

We first prove a lemma.
\begin{lemma} \label{lem:decompose}
    Let $\eta: [2^{Q+1}]\to \B$ be the binary decomposition.
There exists a $(2Q+2)$-layer network $f:\R\to\R^Q$
    with ReLU activations and width at most $Q+1$
    such that for $x\in [2^{Q+1}]$
\begin{align*}
    f(x)=\eta(x).
\end{align*}
\end{lemma}

\begin{proof}
We decompose $x\in [2^{Q+1}]$ as $x=x_0+2x_1+\cdots +2^{Q} x_Q$.
Then  for $q=0,\cdots,Q$, we have
    \begin{align*}
        x_q = \sigma(1-\sigma(2^q+2^{q+1}x_{q+1}+\cdots 2^Qx_Q - x)).
    \end{align*}
Thus, for $q=0,\cdots,Q-1$, we construct
\begin{align*}
    f_{2q+1}(
    \begin{pmatrix}
        x\\
        x_Q \\
        \vdots \\
        x_{Q-q+1}
    \end{pmatrix}
    )
    =
    \sigma(
    \begin{pmatrix}
        0\\
        \vdots \\
        0 \\
        2^{Q-q}
    \end{pmatrix}
    +
    \begin{pmatrix}
        1   &\cdots &0  &0 \\
        \vdots &\ddots &\vdots &\vdots \\
        0   &\cdots &1  &0 \\
        -1 &2^Q &\cdots &2^{Q-q+1}
    \end{pmatrix}
    \begin{pmatrix}
        x\\
        x_Q \\
        \vdots \\
        x_{Q-q+1}
    \end{pmatrix}
    )\in \R^{q+2} ,
\end{align*}

\begin{align*}
    f_{2q+2}(
    \begin{pmatrix}
        x\\
        x_Q \\
        \vdots \\
        x_{Q-q}
    \end{pmatrix}
    )
    =
    \sigma(
    \begin{pmatrix}
        0\\
        \vdots \\
        0 \\
        1
    \end{pmatrix}
    +
    \begin{pmatrix}
        1   &\cdots &0  &0 \\
        \vdots &\ddots &\vdots &\vdots \\
        0   &\cdots &1  &0 \\
        0   &\cdots &0  &-1
    \end{pmatrix}
    \begin{pmatrix}
        x\\
        x_Q \\
        \vdots \\
        x_{Q-q}
    \end{pmatrix}
    )\in \R^{q+2}.
\end{align*}

The last 2 layers
\begin{align*}
    f_{2Q+1}(
    \begin{pmatrix}
        x\\
        x_Q \\
        \vdots \\
        x_{1}
    \end{pmatrix}
    )
    =
    \sigma(
    \begin{pmatrix}
        1\\
        0\\
        \vdots \\
        0 \\
    \end{pmatrix}
    +
    \begin{pmatrix}
        -1 &2^Q&\cdots &2\\
        0  &1  &\cdots &0 \\
        \vdots &\vdots &\ddots &\vdots \\
        0  &0  &\cdots &1
    \end{pmatrix}
    \begin{pmatrix}
        x\\
        x_Q \\
        \vdots \\
        x_1
    \end{pmatrix}
    )\R^{Q+1},
\end{align*}

\begin{align*}
    f_{2Q+2}(
    \begin{pmatrix}
        x_0\\
        x_Q \\
        \vdots \\
        x_1
    \end{pmatrix}
    )
    =
    \sigma(
    \begin{pmatrix}
        1\\
        0\\
        \vdots \\
        0
    \end{pmatrix}
    +
    \begin{pmatrix}
        -1  &0  &\cdots &0\\
        0   &1  &\cdots &0 \\
        \vdots  &\vdots &\ddots &\vdots \\
        0   &0  &\cdots &1
    \end{pmatrix}
    \begin{pmatrix}
        x_0\\
        x_Q \\
        \vdots \\
        x_1
    \end{pmatrix}
    )\in \R^{Q+1}.
\end{align*}
    Let $f=f_{2Q+2}\circ \cdots \circ f_1$.
    For $x\in [2^{Q+1}]$ and $x=x_0+2x_1+\cdots +2^{Q} x_Q$
    we have
    \begin{align*}
        f(x) = \begin{pmatrix}
        x_0\\
        x_Q \\
        \vdots \\
        x_1
    \end{pmatrix}.
    \end{align*}
\end{proof}

\lemapprox*

\begin{proof}
    From Lemma~\ref{lem:decompose}, there exists a network $f$
    such that for $x\in [2^{Q+1}]$, $f(x)=\eta(x)$.
    We denote the $l$-the layer of $f$ by $f_l$ for $l=1,\cdots,2Q+2$.
    Without loss of generality, we assume for $z\in \R^{K_{l-1}}$
    \begin{align*}
        f_l(z) = \sigma(w_l \cdot z + b_l),
    \end{align*}
    where $\sigma$ is ReLU activation and $w_l\in \R^{K_{l}\times K_{l-1}}, b_l\in \R^{K_l}$.

    Now we construct a $(2Q+2)$-layer network $F$
    in the form of Equation~\eqref{equ:network}.
    For $l=1,\cdots,2Q+2$, let $n_l=K_l\times P$.
    We construct $F_l$ in the form of Equation~\eqref{equ:layer} as
    \begin{align*}
        W^{k_l\times p,k_{l-1}\times r}[l] &= \left\{
        \begin{aligned}
            &{\rm{diag}}(w_l [k_l, k_{l-1}]) &\text{ if } p=r \\
            &0                  &\text{otherwise}\\
        \end{aligned}
        \right. ,\\
        C^{k_l\times p,k_{l-1}\times r}[l] &= \left\{
        \begin{aligned}
            &\frac{b_l[k_l]}{K_{l-1}} &\text{ if } p=r \\
            &0                  &\text{otherwise}\\
        \end{aligned}
        \right. ,\\
    \end{align*}
    for $k_l=1,\cdots,K_l$, $k_{l-1}=1,\cdots,K_{l-1}$ and $p,r=1,\cdots,P$.

    We can verify that for $X\in \R^{K_{l-1}\times P\times N}$,
    $k_l=1,\cdots,K_l$, $p=1,\cdots,P$ and $i=0,\cdots, N-1$
    \begin{align*}
        F_l(X)[k_l, p, i]
        &= \sigma(W^{k_l\times p}[l]\cdot X + C^{k_l\times p}[l]\cdot \one)[i]\\
        &= \sigma(\sum_{k_{l-1}=1}^{K_{l-1}} W^{k_l\times p, k_{l-1}\times p}\cdot X[k_{l-1},p,:]+C^{k_l\times p, k_{l-1}\times p})[i]\\
        &= \sigma(\sum_{k_{l-1}=1}^{K_{l-1}} {\rm{diag}}(w_l[k_l, k_{l-1}])\cdot X[k_{l-1},p,:]+ \frac{b_l[k_l]}{K_{l-1}})[i]\\
        &= \sigma(b_l[k_l] + \sum_{k_{l-1}=1}^{K_{l-1}} w_l[k_l, k_{l-1}] \cdot X[k_{l-1},p,i])\\
        &= \sigma(w_l[k_l,:]\cdot X[:, p, i] + b_l[k_l])\\
        &= f_l(X[:,p,i])[k_l].
    \end{align*}
    That is,
    \begin{align*}
        F_l(X)[:,p,i] = f_l(X[:,p,i]).
    \end{align*}
    Thus, for $X\in \R^{P\times N}$ and $p=1\cdots P$, $i=0,\cdots,N-1$
    \begin{align*}
        F(X)[:,p,i] =& f_{2Q+2}(F_{2Q+1}\circ\cdots\circ F_1(X)[:,p,i])\\
         &\vdots \\
        =&f_{2Q+2}\circ\cdots\circ f_1(X[p,i])\\
        =&f(X[p,i]).
    \end{align*}

    For $Z\in [2^{Q+1}]^{P\times N}$,
    \begin{align*}
        F(Z) = \eta(Z).
    \end{align*}
\end{proof}

\section{Proof of Lemma~\ref{lem:binary-exist}}\label{app:binary-exist}
Assume a network $F=F_2\circ F_1$ in the form of Equation~\eqref{equ:network}
with ReLU activations and $n_0=G$, $n_1=S$, $n_2=1$
satisfies that for $X\in G\times N$
\begin{align}
    F(X)=\frac{1}{S}\sum_{s=1}^{S} \sigma(W^s[1] \cdot X + C^s[1] \cdot \one). \label{equ:binary-network}
\end{align}
Here, the weights and biases in $F_2$ degenerate as
\begin{align*}
    W[2] &= W^1[2] = (\frac{1}{S}I, \cdots, \frac{1}{S}I),
    C[2] = 0.
\end{align*}
For convenience,
in the rest of this section
we simplify $W^s[1], C^s[1]$ to $W^s, C^s$.

The following result is well known.
\begin{lemma}\label{proposition}
 Let $\mathcal{B} = \{Z_s|s=1,2,\cdots, S \} \subset \B^{G\times N}$
be a $G$-channel aperiodic binary dataset.
 Let $\Vert Z \Vert$ be the $L_2$-norm of $Z$.
 \begin{itemize}
 \item[1)] $T^{\bm{0}}(Z_s)\cdot Z_s = \Vert Z_s\Vert ^2$.
 \item[2)] $T^M(Z_s)\cdot Z_t \leq\Vert Z_s\Vert ^2\leq  (GN)^2$
            for any $M\in \Z^d$.
         \item[3)] For any $M\in \Z^d$ that $M \modd (n_1, n_2, \cdots, n_d) \neq\zero$, $T^{M}(Z_s)\cdot Z_s \leq \Vert Z_s\Vert ^2 -1$.

         \item[4)] If $\Vert Z_s\Vert  = \Vert Z_t\Vert $, $T^M(Z_s)\cdot Z_t \leq \Vert Z_{t} \Vert^2 -1$ for any $M\in \Z^d$.
 \item[5)] If $\Vert Z_s\Vert  > \Vert Z_t\Vert $,
                    $\Vert Z_s\Vert  \geq \sqrt{\Vert Z_t\Vert ^2 + 1} \geq \Vert Z_t\Vert + \frac{1}{2 GN}$.
 \end{itemize}
\end{lemma}

The $i$-th component of $F(Z)$ in Equation~\eqref{equ:binary-network} is
\begin{align*}
    F(Z)[i] = \frac{1}{S}\sum_{s=1}^{S}\sigma(W_i^k\cdot Z + C^k\cdot \one),
\end{align*}
where
$W_i^k= (W_i^{k,1},\cdots,W_i^{k,G})\in \R^{G\times N}$
and $W_i^{k,r}$ is the $i$-th row of $W^{k, r}$.
Recall that each circular filter $W^{k, r}\in \R^{N\times N}$
in Equation~\eqref{equ:binary-network}
is determined by its first row $W^{k, r}_0\in \R^N$
and $W^{k, r}_{\delta(M)}=T^M(W^{k, r}_0)$.
And the biases $C^{k, r}$ are actually scalars.

\begin{lemma}\label{lem:main}
Let $\mathcal{B} = \{Z_s|s=1,2,\cdots, S \} \subset \B^{G\times N}$
be a $G$-channel aperiodic binary dataset.
Endow $\mathcal{B}$ with an order that
$s\geq t \iff \Vert Z_s\Vert \geq \Vert Z_t\Vert.
$
Construct the filters and biases in equation~(\ref{equ:binary-network}) as
\begin{align*}
W^{s, r}_0 &= \frac{Z_s^r}{\Vert Z_s\Vert},\\
C^{s, r} &= \frac{1}{G(2GN+1)} -\frac{\Vert Z_{s-1}\Vert}{G},
\end{align*}
for $s=1,\cdots S, r=1\cdots,G$ and set $\Vert Z_0\Vert = \frac{1}{2GN+1}$.

Then,
\begin{itemize}
\item [a)] if $t<s$, then $\sigma(W^s_i\cdot Z_t+C^s\cdot \one)=0, i=0,1,\cdots,GN-1$;
\item [b)] if $t=s$, then $\sigma(W^s_0\cdot Z_s + C^s\cdot \one) - \sigma(W^s_i\cdot Z_s + C^s)>\frac{1}{2 GN+1},i=1,2,\cdots ,GN-1$;
\item [c)] if $t>s$, then $\sigma(W^s_i\cdot Z_t+C^s\cdot \one) <GN$.

\end{itemize}

\end{lemma}

\begin{proof}
This proof uses Lemma~\ref{proposition}.
\begin{itemize}
\item[a)] Assuming $t<s$,   we have
\begin{align*}
W^s_{\delta(M)}\cdot Z_t+C^s \cdot \one&= T^M(Z_s)\cdot Z_t / \Vert Z_s\Vert+ \frac{1}{2GN+1} - \Vert Z_{s-1}\Vert,\\
&\leq T^M(Z_s)\cdot Z_t / \Vert Z_s\Vert+ \frac{1}{2GN+1} - \Vert Z_{t}\Vert.
\end{align*}
If $\Vert Z_s\Vert  = \Vert Z_t\Vert $,
\begin{align*}
T^M(Z_s)\cdot Z_t / \Vert Z_s\Vert &\leq \Vert Z_t\Vert - \frac{1}{\Vert Z_t\Vert},\\
W^s_{\delta(M)}\cdot Z_t+C^s \cdot \one& \leq \frac{1}{2GN+1} - \frac{1}{\Vert Z_t\Vert} < 0.
\end{align*}
If $\Vert Z_s\Vert > \Vert Z_t\Vert $,
\begin{align*}
T^M(Z_s)\cdot Z_t / \Vert Z_s\Vert &\leq \frac{\Vert Z_t\Vert^2}{\Vert Z_t\Vert + \frac{1}{2 GN}},\\
W^s_{\delta(M)}\cdot Z_t+C^s\cdot \one &\leq \frac{\Vert Z_t\Vert^2}{\Vert Z_t\Vert + \frac{1}{2 GN}}+\frac{1}{2GN+1} - \Vert Z_{t}\Vert\\
&= \frac{1}{2GN+1}-\frac{1}{2GN+1/\Vert Z_t\Vert}\\
&< 0.
\end{align*}
Thus, for all $M\in \Z^d$,
\begin{align*}
\sigma(W^s_{\delta(M)}\cdot Z_t+C^s\cdot \one)=0.
\end{align*}

\item [b)] We have
\begin{align*}
W^s_0\cdot Z_s + C^s\cdot \one &= \Vert Z_{s}\Vert - \Vert Z_{s-1}\Vert+\frac{1}{2GN+1} ,\\
W^s_{\delta(M)}\cdot Z_s + C^s\cdot \one &= T^M(Z_s)\cdot Z_s / \Vert Z_s\Vert+ \frac{1}{2GN+1} - \Vert Z_{s-1}\Vert\\
&\leq \frac{\Vert Z_s\Vert^2-1}{\Vert Z_s\Vert}+ \frac{1}{2GN+1} - \Vert Z_{s-1}\Vert\\
&= \Vert Z_{s}\Vert- \Vert Z_{s-1}\Vert +\frac{1}{2GN+1}  - \frac{1}{\Vert Z_s\Vert}.
\end{align*}
Since
\begin{align*}
\Vert Z_{s}\Vert- \Vert Z_{s-1}\Vert \geq 0 \text{  and  }
\frac{1}{2GN+1}  - \frac{1}{\Vert Z_s\Vert} < 0,\\
\end{align*}
we have
\begin{align*}
\sigma(\Vert Z_{s}\Vert - \Vert Z_{s-1}\Vert+\frac{1}{2GN+1}) -
\sigma(\Vert Z_{s}\Vert- \Vert Z_{s-1}\Vert +\frac{1}{2GN+1}  - \frac{1}{\Vert Z_s\Vert})
\geq \frac{1}{2GN+1}.
\end{align*}
\item [c)] Assuming $t>s$,   we have
\begin{align*}
W^s_{\delta(M)}\cdot Z_t+C^s\cdot \one &= T^M(Z_s)\cdot Z_t / \Vert Z_s\Vert+ \frac{1}{2GN+1} - \Vert Z_{s-1}\Vert\\
&\leq\Vert Z_s\Vert- \frac{1}{\Vert Z_s\Vert}+ \frac{1}{2GN+1} - \Vert Z_{s-1}\Vert \\
&< \Vert Z_s\Vert\\
&\leq GN,\\
\sigma( W^s_{\delta(M)}\cdot Z_t+C^s\cdot \one) &< GN.
\end{align*}

\end{itemize}

\end{proof}

\lemexistence*

\begin{proof}
Without loss of generality,
we assign an order to the dataset that
\begin{align*}
s\geq t \iff \Vert Z_s\Vert \geq \Vert Z_t\Vert.
\end{align*}

We set $\alpha \geq 1+GN+2G^2N^2$
and construct $F$ as Equation~\eqref{equ:binary-network} such that
\begin{align*}
W^{s, r}_0 &= \frac{\alpha^{s-1} Z_s^r}{\Vert Z_s\Vert},\\
C^{s, r} &= \frac{\alpha^{s-1}}{G(2GN+1)} -\frac{\alpha^{s-1} \Vert Z_{s-1}\Vert}{G},
\end{align*}
for $s=1,\cdots S, r=1\cdots,G$ and set $\Vert Z_0\Vert = \frac{1}{2GN+1}$.

From Lemma~\ref{lem:main},
we have for $i=1,2,\cdots, GN-1$,
\begin{align*}
&S(F(Z_t)[0] -  F(Z_t)[i])\\
=& \sum_{s=1}^S \alpha^{s-1} [\sigma(W^s_0\cdot Z_t + C^s) - \sigma(W^s_i\cdot Z_t + C^s)]\\
=& \sum_{s=1}^{t} \alpha^{s-1} [\sigma(W^s_0\cdot Z_t + C^s) - \sigma(W^s_i\cdot Z_t + C^s)]\\
\geq & \frac{\alpha^{t-1}}{2GN+1} + \sum_{s=1}^{t-1} \alpha^{s-1} [\sigma(W^s_0\cdot Z_t + C^s) - \sigma(W^s_i\cdot Z_t + C^s)]\\
> & \frac{\alpha^{t-1} }{2GN+1} - GN\sum_{s=1}^{t-1} \alpha^{s-1} \\
=& \frac{\alpha^{t-1}}{2GN+1} - \frac{GN(1-\alpha^{t-1})}{1-\alpha}\\
=& \frac{(\alpha -2G^2N^2-GN-1)\alpha^{t-1}+2G^2N^2+GN}{(2GN+1)(\alpha-1)}\\
>& 0.
\end{align*}
\end{proof}

\section{Experimental settings}\label{app:exp}
    For CIFAR-10, we constantly pad 4 pixels with values 0 around images.
    For MNIST, we resize images to $32\times 32$.
    For 3D-MNIST, we voxelize this dataset and constantly pad 8 pixels with value 0 around images.

    We leverage restorers with 6 layers.
    In each layer, we use a sparse circular filter, for example, its kernel size is 9.
    Each layer outputs only one channel and has no bias parameter.

    For rotation restoration for MNIST, we set $n_1=n_2=36$, $R=112$, and $a=0.92$.

\end{document}